\documentclass[a4paper]{article}
\usepackage[margin=1.2in]{geometry}
\usepackage{amsmath}
\usepackage{algorithmic}
\usepackage{algorithm}
\usepackage{amsthm}
\usepackage{amsfonts}
\usepackage{comment}
\usepackage{graphicx}
\usepackage{hyperref}
\usepackage{color}

\newcommand{\eqdef}{\stackrel{\Delta}{=}}
\newcommand{\order}{{\bf order}}

\newcommand{\ceil}[1]{\left\lceil #1 \right \rceil}

\def\reals{{\mathcal R}}

\newcommand{\mE}{\mathcal{E}}

\newcommand{\K}{\mathcal{K}}
\newcommand{\R}{\mathcal{R}}

\newcommand{\E}{\mathop{\mbox{\bf E}}}

\newcommand{\conv}{\mathop{\mbox{\rm conv}}}

\newcommand{\poly}{\mathop{\mbox{\rm poly}}}

\newcommand{\ignore}[1]{}

\def\trace{{\bf Tr}}

\def\reals{{\mathbb R}}

\def\lhat{\hat{\ell}}

\def\ball{{\mathbb B}_n}
\def\bold0{\mathbf{0}}

\def\bone{\mathbf{1}}

\def\trace{{\bf Tr}}

\newcommand{\eps}{\varepsilon}

\newtheorem{theorem}{Theorem}[section]
\newtheorem{definition}{Definition}[section]
\newtheorem{lemma}{Lemma}[section]
\newtheorem{observation}{Observation}[section]
\newtheorem{corollary}{Corollary}[section]

\newcommand{\ratiospanner}{ratio-volumetric spanner\ }
\newcommand{\probspanner}{exp-volumetric spanner\ }
\newcommand{\eat}[1]{}

\title{Volumetric Spanners: \\ an Efficient Exploration Basis for Learning  }

\author{Elad Hazan \\
\small{Technion - Israel Inst. of Tech.} \\ 
\small{ehazan@ie.technion.ac.il}
\and Zohar Karnin\\  
\small{Yahoo Labs} \\
\small{zkarnin@ymail.com} 
\and Raghu Meka \\
\small{Microsoft Research} \\ 
\small{meka@microsoft.com}
}

\date{}
\begin{document} 
\maketitle

\begin{abstract}
Numerous machine learning problems require an {\it exploration basis} - a mechanism to explore the action space. We define a novel geometric notion of exploration basis with low variance called volumetric spanners, and give efficient algorithms to construct such bases. 

We show how efficient volumetric spanners give rise to an efficient and near-optimal regret algorithm for bandit linear optimization over general convex sets. Previously such results were known only for specific convex sets, or under special conditions such as the existence of an efficient self-concordant barrier for the underlying set. 
\end{abstract}

\section{Introduction}
A fundamental difficulty in machine learning  is environment exploration. A prominent example is the famed multi-armed bandit (MAB) problem, in which a decision maker iteratively chooses an action from a set of  available actions and receives a payoff, without observing the payoff of all other actions she could have taken. The MAB problem displays an exploration-exploitation tradeoff, in which the decision maker trades exploring the action space vs. exploiting the knowledge already obtained to pick the best arm. 

Another example in which environment exploration is crucial, or perhaps the main point, is active learning and experiment design. In these fields it is important to correctly identify the most informative queries so as to efficiently construct a solution. 

Exploration is hardly summarized by picking an action uniformly at random \footnote{although uniform sampling does at times work exceptionally well, e.g. \cite{Cesa-BianchiL12} }. Indeed, sophisticated techniques from various areas of optimization, statistics and convex geometry have been applied to designing ever better exploration algorithms.  To mention a few:    \cite{AwerbuchK08} devise the notion of {\it barycentric spanners}, and use this construction to give the first low-regret algorithms for complex decision problems such as online routing.   \cite{AHR12} use self-concordant barriers to build an efficient exploration strategy for convex sets in Euclidean space. \cite{BubeckCK12} apply tools from convex geometry, namely the John ellipsoid to construct optimal-regret algorithms for bandit linear optimization, albeit not always efficiently. 

In this paper we consider a generic approach to exploration, and quantify what efficient exploration with {\sl low variance} requires in general. Given a set in Euclidean space, a low-variance exploration basis  is a subset with the following property: given  noisy estimates of a linear function over the basis, one can construct an estimate for the linear function over the entire set without increasing the variance of the estimates. 

By definition, such low variance exploration bases are immediately applicable to noisy linear regression: given a low-variance exploration basis, it suffices to learn the function values only over the basis in order to interpolate the value of the underlying linear regressor over the entire decision set. This fact can be used for active learning as well as for the exploration component of a bandit linear optimization algorithm.

Henceforth we define a novel construction for a low variance exploration basis called {\bf volumetric spanners} and give efficient algorithms to construct them. We further investigate the convex geometry implications of our construction, and define the notion of a {\bf minimal volumetric ellipsoid} of a convex body. We give structural theorems on the existence and properties of these ellipsoids, as well as constructive algorithms to compute them in several cases. 

We complement our findings with an  application to machine learning, in which we advance a well-studied open problem that has exploration as its core difficulty: an efficient and near-optimal  regret algorithm for bandit linear optimization (BLO). We expect that volumetric spanners and volumetric ellipsoids can be useful elsewhere in experiment design and active learning. We briefly discuss the application to BLO next. 

\paragraph{Bandit Linear Optimization}
Bandit linear optimization (BLO) is a fundamental problem in decision making under uncertainty that efficiently captures structured action sets. The canonical example is that of online routing in graphs: a decision maker iteratively chooses a path in a given graph from source to destination, the adversary chooses lengths of the edges of the graph, and the decision maker receives as feedback the length of the path she chose but no other information (see \cite{AwerbuchK08}). Her goal over many iterations is to attain an average travel time as short as that of the best fixed shortest path in the graph. 

This decision problem is readily modeled in the ``experts" framework, albeit with efficiency issues: the number of possible paths is potentially exponential in the graph representation. The BLO framework gives an efficient model for capturing such structured decision problems: iteratively a decision maker chooses a point in a convex set and receives as a payoff an adversarially chosen linear cost function.  In the particular case of online routing, the decision set is taken to be the s-t-flow polytope, which captures the convex hull of all source-destination shortest paths in a given graph, and has a succinct representation with polynomially many constraints and low dimensionality. The linear cost function corresponds to a weight function on the graphs edges, where the length of a path is defined as the sum of weights of its edges.

The BLO framework captures many other structured problems efficiently, e.g., learning permutations,  rankings and other examples (see \cite{AHR12}). As such, it has been the focus of much research in the past few years. The reader is referred to the recent survey of  \cite{jbubeck12} for more details on algorithmic results for BLO. Let us remark that certain online bandit problems do not immediately fall into the convex BLO model that we address, such as combinatorial bandits studied in \cite{Cesa-BianchiL12}.  

In this paper we contribute to the large literature on the BLO model by giving the first polynomial-time  and near optimal-regret algorithm for BLO over  general convex decision sets; see Section \ref{sec:blo} for a formal statement . Previously efficient algorithms, with non-optimal-regret, were known over convex sets that admit an efficient self-concordant barrier \cite{AHR12}, and optimal-regret algorithms were known over general sets \cite{BubeckCK12} but these algorithms were not computationally efficient. Our result, based on volumetric spanners, is able to attain the best of both worlds.

\subsection{Volumetric Ellipsoids and Spanners}
We now describe the main convex geometric concepts we introduce and use for low variance exploration. To do so we first review some basic notions from convex geometry.

Let $\reals^d$ be the $d$-dimensional vector space over the reals. Given a set of vectors $S = \{v_1,\ldots,v_t\} \subset \reals^d$, we denote by $\mE(S)$ the ellipsoid defined by $S$:
$$ \mathcal{E}(S) = \left\{ \sum_{i \in S} \alpha_i v_i \;:\; \ \sum_i \alpha_i^2 \leq 1 \right\}.$$
By abuse of notation, we also say that $\mE(S)$ is \emph{supported} on the set $S$. 

Ellipsoids play an important role in convex geometry and specific ellipsoids associated with a convex body have been used in previous works in machine learning for designing good exploration bases for convex sets $\K \subseteq \reals^d$
. For example, the notion of \emph{barycentric spanners} which were introduced in the seminal work of \cite{AwerbuchK08} corresponds to looking at the ellipsoid of maximum volume supported by exactly $d$ points from $\K$\footnote{While the definition of \cite{AwerbuchK08} is not phrased as such, their analysis shows the existence of barycentric spanners by looking at the maximum volume ellipsoid.}. Barycentric Spanners have since been used as an exploration basis in several works: In \cite{dani2007price} for online bandit linear optimization, in \cite{bartlett2008high} for a high probability counterpart of the online bandit linear optimization, in \cite{kakade2009playing} for repeated decision making of approximable functions and in \cite{dani2008stochastic} for a stochastic version of bandit linear optimization. Another example is the work of \cite{BubeckCK12} which looks  at the minimum volume enclosing ellipsoid (MVEE) also known as the John ellipsoid (see Section \ref{sec:prelims} for more background on this fundamental object from convex geometry).

As will be clear soon, our definition of a \emph{minimal volumetric ellipsoid} enjoys the best properties of the examples above enabling us to get more efficient algorithms. Similar to barycentric spanners, it is supported by a small (quasi-linear) set of points of $\K$. Simultaneously and unlike the barycentric counterpart, the volumetric ellipsoid contains the body $\K$, a property shared with the John ellipsoid. 

\begin{definition}[Volumetric Ellipsoids]\label{def:orderset}
Let $\K \subseteq \reals^d$ be a  set in Euclidean space. For $S \subseteq \K$, we say that $\mE(S)$ is a \emph{volumetric ellipsoid} for $\K$ if it contains $\K$. We say that $\mE_\K = \mE(S)$ is a  \emph{minimal volumetric ellipsoid}  if it is a containing ellipsoid defined by a set of minimal cardinality
\begin{eqnarray*}
\mE_\K \in  \min_{|S|} \left\{ \mE(S) \ \mbox{ such that } \ S \subseteq \K \subseteq  \mE(S)  \right\}. 
\end{eqnarray*}
We say that $|S|$ is the \emph{order} of the minimal volumetric ellipsoid or of \footnote{We note that our definition allows for multi-sets, meaning that $S$ may contain the same vector more than once} $\K$ denoted $\order(\K)$.
\end{definition}

We discuss various geometric properties of volumetric ellipsoids later. For now, we focus on their utility in designing efficient exploration bases. 
To make this concrete and to simplify some terminology later on (and also to draw an analogy to barycentric spanners), we introduce the notion of {\sl volumetric spanners}. Informally, these correspond to sets $S$ that span all points in a given set with coefficients having Euclidean norm at most one. Formally:
\begin{definition}
Let $\K \subseteq \reals^d$ and let $S \subseteq \K$. We say that $S$ is a \emph{volumetric spanner} for $\K$ if  $\K \subseteq \mE(S)$.
\end{definition}

Clearly, a set $\K$ has a volumetric spanner of cardinality at most $t$ if and only if $\order(\K) \leq t$. \\

 Our goal in this work is to bound the order of arbitrary sets. A priori, it is not even clear if there is a universal bound (depending only on the dimension and not on the set) on the order $S$ for compact sets $\K$. However, barycentric spanners and the John ellipsoid show that the order of any compact set in $\R^d$ is at most $O(d^2)$. Our main structural result in convex geometry gives a nearly optimal linear bound on the order of all sets.
\begin{theorem}\label{th:mainorder}
Any compact set $\K \subseteq \reals^d$ admits a volumetric ellipsoid of order at most $12d$.  
Further, if $\K = \{v_1,\ldots,v_n\}$ is a discrete set, then a volumetric ellipsoid for $\K$ of order at most $12d$ can be constructed in time $O(n^{3.5}+dn^3+d^5)$. 
\end{theorem}
We emphasize the last part of the above theorem giving an algorithm for finding volumetric spanners of small size; this could be useful in using our structural results for algorithmic purposes. We also give a different algorithmic construction for the discrete case (a set of $n$ vectors) in Section~\ref{sec:algdiscrete}. While being sub-optimal by logarithmic factors (gives an ellipsoid of order $O(d (\log d)(\log n))$ this alternate construction has the advantage of being simpler and more efficient to compute. 

\subsection{Approximate Volumetric Spanners}
Theorem \ref{th:mainorder} shows the existence of good volumetric spanners and also gives an efficient algorithm for finding such a spanner in the discrete case, i.e.\ when $\K$ is finite and given explicitly. However, the existence proof uses the John ellipsoid in a fundamental way and it is not known how to compute (even approximately) the John ellipsoid efficiently for the case of general convex bodies. For such computationally difficult cases, we introduce a natural relaxation of volumetric ellipsoids which can be computed efficiently for a bigger class of bodies and is similarly useful. The relaxation comes from requiring that the ellipsoid of small support contain all but an $\eps$ fraction of the points in $\K$ (under some distribution). In addition, we also require that the measure of points decays exponentially fast w.r.t their $\mE(S)$-norm (see precise definition in next section); this property gives us tighter control on the set of points not contained in the ellipsoid. When discussing a measure over the points of a body the most natural one is the uniform distribution over the body. However, it makes sense to consider other measures as well and our approximation results in fact hold for a wide class of distributions.
\begin{definition} \label{def:volume apx}
Let $S \subseteq \R^d$ be a set of vectors and let $V$ be the matrix whose columns are the vectors of $S$. We define the semi-norm 
$$\|x\|_{\mE(S)}=\sqrt{x^\top (VV^\top)^{-1} x} \ ,$$ where $(VV^\top)^{-1}$ is the Moore-Penrose pseudo-inverse of $VV^\top$.
Let $\K$ be a convex set in $\R^d$ and $p$ a distribution over it. Let $\eps>0$.
A $(p,\eps)$-\emph{\probspanner}of $\K$ is a set $S \subseteq \K$ such that for any $\theta>1$
$$ \Pr_{x \sim p}[\|x\|_{\mE(S)} \geq \theta ] \leq \eps^{-\theta}.  $$
\end{definition}

We prove that spanners as above can be efficiently obtained for any log-concave distribution: 

\begin{theorem} \label{thm:vol apx}
Let $\K$ be a convex set in $\R^d$ and $p$ a log-concave distribution over it.
By sampling $O(d+\log^2(1/\eps))$ i.i.d.\ points from $p$ one obtains, w.p.\ at least $1-\exp\left(-\sqrt{\max\left\{ \log(1/\eps),d \right\}}  \right)$
, a $(p,\eps)$-\probspanner for $\K$. 
\end{theorem}

\eat{For the second definition we describe a spanner that covers all but an $\eps$ fraction of the points in $\K$ and moreover, the measure of the points decays exponentially fast w.r.t their $\mE(S)$-norm. Since we are discussing a measure over the points of a body it makes sense not only to consider a uniform distribution over the body but an arbitrary one. As it turns out, the approximation can be efficiently obtained for any log-concave distribution.
\begin{definition} \label{def:volume apx}
Let $\K$ be a body in $\R^d$ and $p$ a distribution over it. Let $\eps>0$.
A $(p,\eps)$-\probspanner of $\K$ is a set $S \subseteq \K$ where for any $\theta>1$
$$ \Pr_{x \sim p}[\|x\|_{\mE(S)} \geq \theta ] \leq \eps^{-\theta}  $$
\end{definition}}


\subsection{Structure of the paper}

In the next section we list the preliminaries and known results from measure concentration, convex geometry and online learning that we need. In Section \ref{sec:existence} we show the existence of small size volumetric spanners. In sections \ref{sec:algcont} and \ref{sec:algdiscrete} we give efficient constructions of volumetric spanners for continuous and discrete sets, respectively. We then proceed to describe the application of our geometric results to bandit linear optimization in Section \ref{sec:blo}.

\section{Preliminaries}\label{sec:prelims}
We now describe several preliminary results we need from convex geometry and linear algebra. We start with some notation:
\begin{itemize}
\item A matrix $A \in \reals^{d \times d}$ is positive semi-definite (PSD) when for all $x \in \R^d$ it holds that $x^\top A x \geq 0$. Alternatively, when all of its eigenvalues are non-negative. We say that $A \succeq B$ if $A-B$ is PSD.
\item Given an ellipsoid $\mE(S) = \{\sum_i \alpha_i v_i \,:\, \sum_i \alpha_i^2 \leq 1\}$, we shall use the notation $\|x\|_{\mE(S)} \eqdef \sqrt{x^\top (VV^\top)^{-1} x}$ to denote the (Minkowski) semi-norm defined by the ellipsoid, where $V$ is the matrix with the vectors $v_i$'s as columns.
\item Throughout, we denote by $I_d$ the $d \times d$ identity matrix.
\end{itemize}
We next describe properties of the John ellipsoid which plays an important role in our proofs.
\subsection{The Fritz John Ellipsoid}
Let $\K \subseteq \reals^n$ be an arbitrary convex body. Then, the {\bf John ellipsoid} of $\K$ is the minimum volume ellipsoid containing $\K$. This ellipsoid is unique and its properties have been the subject of important study in convex geometry since the seminal work of  \cite{John48} (see \cite{Ball97} and \cite{henk} for historic information). 

Suppose that we have linearly transformed $\K$ such that its minimum volume enclosing ellipsoid (MVEE) is the unit sphere; in convex geometric terms, $\K$ is in \emph{John's position}.  
\eat{``John's ellipsoid and in particular its contact points with the convex body it encapsulates makes for an appealing exploration basis. Indeed \cite{BubeckCK12} have used exactly this machinery to attain an optimal-regret algorithm for BLO. Unfortunately we know of no efficient algorithm to compute, or even approximate up to a constant, the John ellipsoid for a general convex set, thus the latter result does not give a polynomial time algorithm for BLO.''}
The celebrated decomposition theorem of  \cite{John48} gives a characterization of when a body is in John's position and will play an important role in our constructions (the version here is from \cite{Ball97}). 

Below we consider only symmetric convex sets, i.e. sets that admit the following property:  if $x \in \K$ then also $-x \in \K$. The sets encountered in machine learning applications are most always symmetric, since estimating a linear function on a point $x$ is equivalent to estimating it on its polar $-x$, and negating the outcome. 

\begin{theorem}[\cite{Ball97}]\label{thm:MVEE decomp}
Let $\K \in \R^d$ be a symmetric 
set such that its MVEE is the unit sphere. Then there exist $m \leq d(d+1)/2-1$ \emph{contact points} of $\K$ and the sphere $u_1,\ldots,u_m$ and non-negative weights $c_1,\ldots,c_m$ such that $\sum_i c_i u_i = 0$ and $\sum c_i u_i u_i^T = I_d$.
\end{theorem}

The contact points of a convex body have been extensively studied in convex geometry and they also make for an appealing exploration basis in our context. Indeed, \cite{BubeckCK12} use these contact points to attain an optimal-regret algorithm for BLO. Unfortunately we know of no efficient algorithm to compute, or even approximate, the John ellipsoid for a general convex set, thus the results of \cite{BubeckCK12} do not yield efficient algorithms for BLO. \eat{The computation of the linear transformation that makes the unit sphere to be the minimum volume enclosing ellipsoid (MVEE) is not known to be efficient in general, nor are the contact points known to be efficiently computable.} 

For our construction of volumetric spanners we need to compute the MVEE of a discrete symmetric set, which can be done efficiently. We make use of the following (folklore) result:

\begin{theorem} [folklore, see e.g. \cite{khachiyan1996rounding,APT08}]
Let $\K \subseteq \R^d$ be a set of $n$ points. It is possible to compute an $\eps$-approximate MVEE for $\K$ (an enclosing ellipsoid of volume at most $(1+\eps)$ that of the MVEE) that is also supported on a subset of $\K$ in time  $O(n^{3.5}  \log \frac{1}{\eps})$.
\end{theorem}

The run-time above is attainable via the ellipsoid method or path-following interior point methods (see references in theorem statement). An approximation algorithm rather than an exact one  is necessary in a real-valued computation model, and the logarithmic dependence on the approximation guarantee is as good as one can hope for in general. 

\eat{Note that this theorem addresses the case of a discrete set of points and/or their convex hull, rather than general convex sets. For our purposes, henceforth it suffices to consider this case; we use a different methodology to build an exploration basis for general convex sets.}

The above theorem allows us to efficiently compute a linear transformation such that the MVEE of $\K$ is essentially the  unit sphere. We can then use linear programming to compute an approximate decomposition like in John's theorem as follows.

\begin{theorem} \label{thm:johnconstructive}
Let $\{x_1,\ldots,x_n\}=\K \subseteq \R^d$ be a set of $n$ points and assume that:
\begin{enumerate}
\item
$\K$ is symmetric. 
\item
The  John Ellipsoid of $\K$ is the unit sphere.
\end{enumerate}
Then it is possible, in $O((\sqrt{n}+d) n^3)$ time, to compute non-negative weights $c_1,\ldots,c_n$ such that (1) $\sum_i c_i x_i = 0$ and (2) $\sum_{i=1}^n c_i x_ix_i^\top = I_d$. 

\end{theorem}
\begin{proof} 

Denote the MVEE of $\K$ by  $\cal E$ and let $V$ be its corresponding $d \times d$ matrix, meaning $V$ is such that $\|y\|_{\cal E}^2 = y^\top V^{-1} y \leq 1$ for all $y \in \K$. By our assumptions $I_d = V $.

As $\K$ is symmetric and its MVEE is the unit ball, according to Theorem~\ref{thm:MVEE decomp}, there exist $m \leq d(d+1)/2-1$ contact points $u_1,\ldots,u_m$ of $\K$ with the unit ball and a vector $c' \in \R^m$ such that $c' \geq 0$, $\sum c_i' = d$ and $\sum c_i' u_i u_i^T = I_d$. It follows that the following LP has a feasible solution: Find $c \in \R^n$ such that $c \geq 0$, $\sum c_i \leq d$ and $ \sum c_i u_i u_i^T =  I_d$.
The described LP has $O(n + d^2)$ constraints and $n$ variables. It can thus be solved in time $O(d+ \sqrt{n})n^3)$ via interior point methods.
\end{proof}

We next state the results from probability theory that we need.

\subsection{Distributions and Measure Concentration}
For a set $\K$, let $x \sim \K$ denote a uniformly random vector from $\K$. 

\begin{definition}
A distribution over $\R^d$ is log-concave if its probability distribution function (pdf) $p$ is such that for all $x,y \in \R^d$ and $\lambda \in [0,1]$, 
$$ p(\lambda x + (1-\lambda)y ) \geq p(x)^\lambda p(y)^{1-\lambda} $$
\end{definition}

Two log-concave distributions of interest to us are (1) the uniform distribution over a convex body and (2) a distribution over a convex body where $p(x) \propto \exp(L^\top x)$, where $L$ is some vector in $\R^d$. The following result shows that given oracle access to the pdf of a log-concave distribution we can sample from it efficiently. An oracle to a pdf accepts as input a point $x \in \R^d$ and returns the value $p(x)$. 

\begin{lemma} [\cite{lovasz2007geometry}, Theorems~2.1 and 2.2]	\label{lem:log-conv-sample}
Let $p$ be a log-concave distribution over $\reals^d$ and let $\delta>0$. Then, given oracle access to $p$, i.e.\ and oracle computing its pdf for any point in $\reals^d$, there is an algorithm which approximately samples from $p$ such that: 
\begin{enumerate}
\item The total variation distance between the produced distribution and the distribution defined by $p$ is no more than $\delta$. That is, the difference between the probabilities of any event in the produced and actual distribution is bounded by  $\delta$.
\item The algorithm requires a pre-processing time of $\tilde{O}(d^5)$.
\item After pre-processing, each sample can be produced in time $\tilde{O}(d^4/\delta^4)$, or amortized time of $\tilde{O}(d^3/\delta^4)$ if more than $d$ samples are needed.
\end{enumerate}
\end{lemma}

\begin{definition}[Isotropic position]
A random variable $x$ is said to be in isotropic position (or isotropic) if
$$\E[x]=0, \ \ \ \ \ \  \E[ x x^\top] = I_d.$$
A set $\K \subseteq \reals^d$ is said to be in isotropic position if $x \sim \K$ is isotropic. Similarly, a distribution $p$ is isotropic if $x \sim p$ is isotropic.
\end{definition}

Henceforth we use several results regarding the concentration of log-concave isotropic random vectors. We use slight modification where the center of the distribution is not necessarily in the origin. For completeness we present the proof of the modified theorems in Appendix~\ref{app:logcon concentrate}

\begin{theorem}[Theorem 4.1 in~\cite{ALPJ10}]\label{thm:logconcave}
Let $p$ be a log-concave distribution over $\R^d$ in isotropic position. There is a constant $C$ such that for all $t , \delta > 0$, the following holds for $n = \frac{C t^4 d \log^2(t/\delta)   }{\delta^2}$. For independent random vectors $x_1,\ldots,x_n \sim p$, with probability at least $1-\exp(-t\sqrt{d})$, 
$$ \left\| \frac{1}{n} \sum_{i=1}^n x_i x_i^\top - I_d \right\| \leq \delta.$$
\end{theorem}

\begin{corollary}\label{cor:logconcave}
Let $p$ be a log-concave distribution over $\R^d$ and $x \sim p$. Assume that $x$ is such that $\E[xx^T]=I_d$. Then, there is a constant $C$ such that for all $t\geq 1 , \delta > 0$, the following holds for $n = \frac{C t^4 d \log^2(t/\delta)   }{\delta^2}$. For independent random vectors $x_1,\ldots,x_n \sim p$, with probability at least $1-\exp(-t\sqrt{d})$, 
$$ \left\| \frac{1}{n} \sum_{i=1}^n x_i x_i^\top - I_d \right\| \leq \delta.$$
\end{corollary}

\begin{theorem}[Theorem 1.1 in \cite{GuedonM11}]\label{thm:thinshell}
There exist constants $c,C$ such that the following holds. Let $p$ be a log-concave distribution over $\R^d$ in  isotropic position and let $x \sim p$.
Then, for all $\theta \geq 0$,
$$ Pr\left[ \left|\|x\| - \sqrt{d}\right| > \theta  \sqrt{d}\right] \leq C \exp(-c \sqrt{d} \cdot \min(\theta^3,\theta)).$$
\end{theorem}

\begin{corollary} \label{cor:quasi_iso}
Let $p$ be a log-concave distribution over $\R^n$ and let $x\sim p$. Assume that $\E[xx^T] = I_d$. Then for some universal $C,c$ it holds for any $\theta \geq 3$ that
$$ \Pr\left[\|x\| > \theta \sqrt{d} \right] \leq C\exp\left(-c \theta \sqrt{d}\right) $$
\end{corollary}

The following theorem provides a concentration bound for random vectors originating from an arbitrary distribution.

\begin{theorem} [\cite{Rudelson99}]\label{th:rudelson}
Let $X$ be a vector-valued random variable over $\reals^d$ with $\E[XX^\top] = \Sigma$ and $\|\Sigma^{-1/2} X\|^2 \leq R$. Then, for independent samples $X_1,\ldots,X_M$ from $X$, and $M \geq C R\log(R/\eps)/\epsilon^2$ the following holds with probability at least $1/2$:
$$\left\| \frac{1}{M} \sum_{i=1}^M X_i - \Sigma \right\| \leq \epsilon \|\Sigma\|.$$   
\end{theorem}

Finally, we also make use of barycentric spanners in our application to BLO and we briefly describe them next.
\subsection{Barycentric Spanners}
\begin{definition}[\cite{AwerbuchK08}]\label{def:barycentric}
A barycentric spanner of $\K \subseteq \R^d$ is a set of $d$ points $S = \{u_1,\ldots,u_d\} \subseteq \K$ such that any point in $\K$ may be expressed as a linear combination of the elements of $S$ using coefficients in $[-1,1]$. For $C>1$, $S$ is a $C$-approximate barycentric spanner of $\K$ if any point in $\K$ may be expressed as a linear combination of the elements of $S$ using coefficients in $[-C,C]$
\end{definition}

In  \cite{AwerbuchK08} it is shown that any compact set has a barycentric spanner. Moreover, they show that given an oracle with the ability to solve linear optimization problems over $\K$, an approximate barycentric spanner can be efficiently obtained. In the following sections we will use this constructive result.

\begin{theorem} [Proposition 2.5 in \cite{AwerbuchK08}] \label{thm:barycentric}
Let $\K$ be a compact set in $\R^d$ that is not contained in any proper linear subspace. Given an oracle for optimizing linear functions over $\K$, for any $C>1$, it is possible to compute a $C$-approximate barycentric spanner for $\K$, using $O(d^2\log_C(d))$ calls to the optimization oracle.
\end{theorem}

\section{Existence of Volumetric Ellipsoids and Spanners}\label{sec:existence}
In this section we show the existence of low order volumetric ellipsoids proving our main structural result, Theorem \ref{th:mainorder}. Before we do so, we first state a few simple properties of volumetric ellipsoids (recall the definition of $\order$ from Definition \ref{def:orderset}):

\begin{itemize}
\item The definition of $\order$ is linear invariant: for any invertible linear transformation $T:\reals^d \to \reals^d$ and $K \subseteq \reals^d$, $\order(\K) = \order(T\K)$. 
\begin{proof}
Let $S \subseteq \K$ be such that $\K \subseteq \mE(S)$. Then, clearly $T\K \subseteq \mE(TS)$. Thus, $\order(T\K) \leq \order(\K)$. The same argument applied to $T^{-1}$ and $T\K$ shows that $\order(\K) \leq \order(T\K)$. 
\end{proof}

\item The minimal volumetric ellipsoid is not unique in general; see example in figure \ref{fig:ellipsoid}. Further, it is in general different from the John ellipsoid. 
\item  For non-degenerate bodies $\K$, their order is naturally lower bounded by $d$, and there are examples in which it is strictly more than $d$ (e.g., figure \ref{fig:ellipsoid}).
\end{itemize}

\begin{figure}[t] 
\centering
\begin{minipage}[b]{0.45\linewidth}
\includegraphics[width=0.6\textwidth]{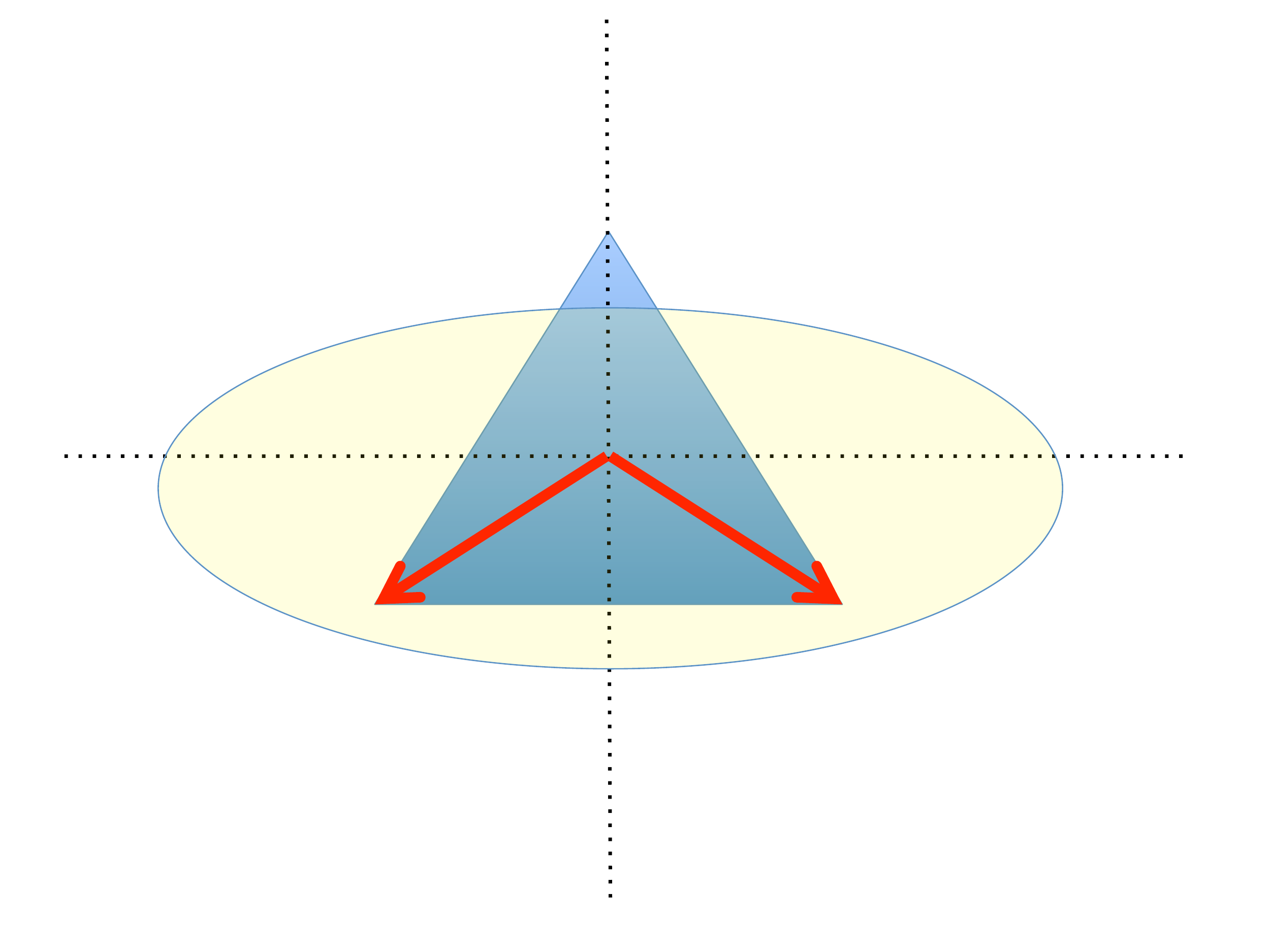}
\end{minipage}
\quad
\begin{minipage}[b]{0.45\linewidth}
\includegraphics[width=0.6\textwidth]{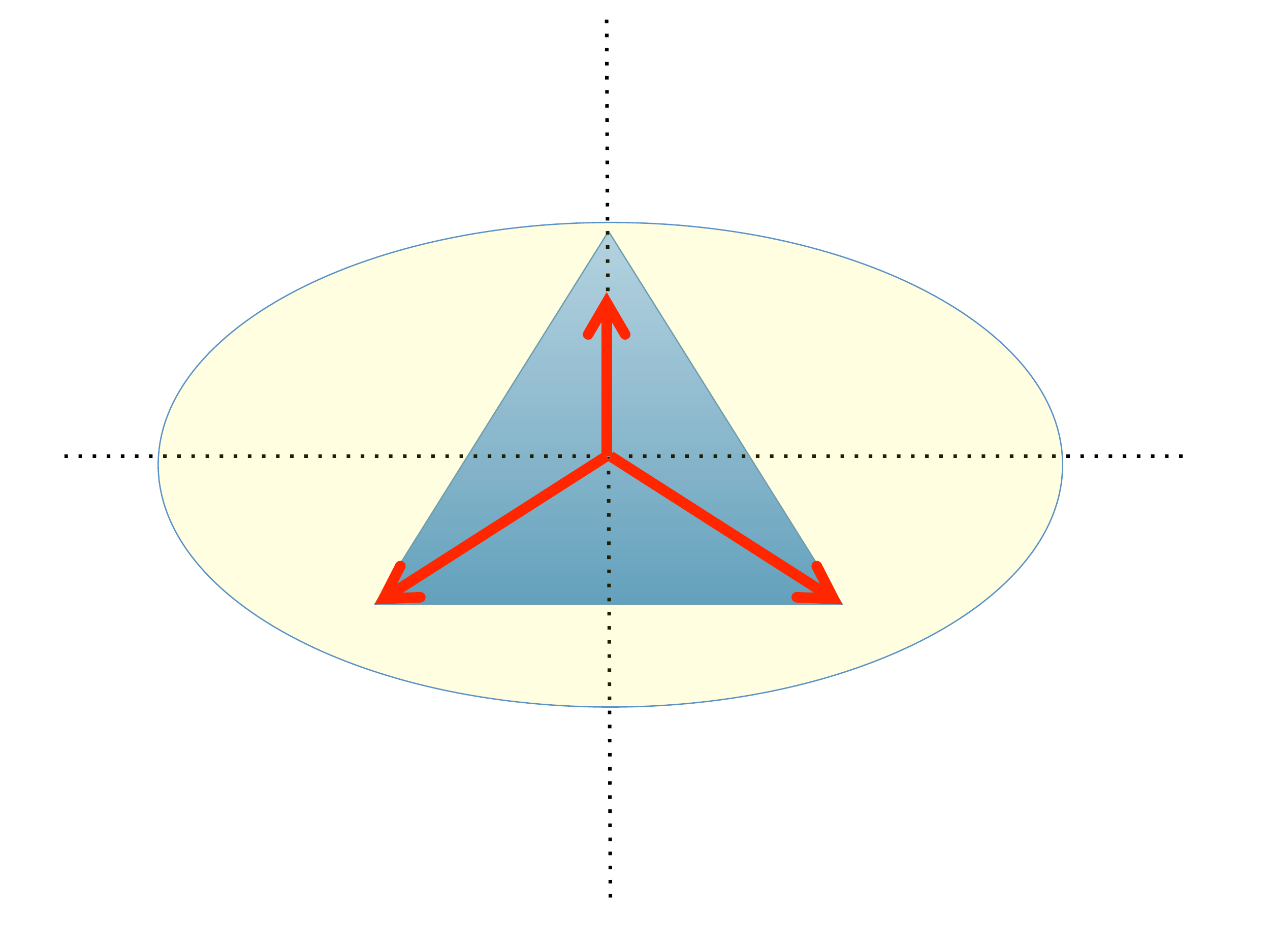}
\end{minipage}
\caption{In $\reals^2$  the order of the volumetric ellipsoid of the equilateral triangle centered at the origin is at least 3. If the vertices are $[0,1], [-\frac{\sqrt{3}}{2}, -\frac{1}{2}], [\frac{\sqrt{3}}{2},-\frac{1}{2}]$, then the eigenpoles of the ellipsoid of the bottom two vertices are $[0.\frac{2}{3}],[2,0]$. The second figure shows one possibility for a volumetric ellipsoid by adding $\frac{3}{4}$ of the first vertex to the previous ellipsoid. This shows the ellipsoid to be non-unique, as it can be rotated three ways.}
\label{fig:ellipsoid}
\end{figure}


To prove Theorem \ref{th:mainorder} we shall use John's decomposition theorem, Theorem \ref{thm:MVEE decomp}.\\

\begin{proof}{{\bf of Theorem \ref{th:mainorder}}}

Let $\K \subseteq \reals^d$ be a compact set. 
Without loss of generality assume that $\K$ is symmetric and contains the origin; we can do so as in the following we only look at outer products of the form $v v^\top$ for vectors $v \in \K$. Further, as $\order(\K)$ is invariant under linear transformations, we can also assume that $\K$ has been linearly transformed to be in John's position.

\eat{
Recall from preceding sections that the John ellipsoid for set $\K$ is the minimum volume ellipsoid that contains $\K$. Henceforth, assume that we have linearly transformed $\K$ such that this largest contained ellipsoid is the unit ball.  Then John's theorem says that $ \frac{1}{\sqrt{d} } \mathbb{B}_d \subseteq \K \subseteq  \mathbb{B}_d $.}
Now, by Theorem \ref{thm:MVEE decomp}, there are $m = O(d^2)$ points $S = \{u_1,\ldots,u_m\}$ on the enclosing unit ball that intersect $\K$ (on the boundary of $\K$), and non-negative weights $c_1,\ldots,c_m$ such that $ \sum_{i \in S}  c_i u_i u_i^\top = I_d$. This implies by taking trace on both sides that $\sum_i c_i = d$. 

Our idea is to start with the vectors $u_1,\ldots,u_m$ as a starting point for a volumetric spanner. However, this set has $O(d^2)$ points which is larger than what we want. We now prune these contact points via the sparsification methods of \cite{batson2012twice}. We use the following lemma that is a corollary of their Theorem 3.1.

\begin{lemma} [\cite{batson2012twice}] \label{lem:pruning contact}
Let $u_1,\ldots,u_m$ be unit vectors and let $p \in \Delta(m)$ be a distribution over $[m]$ such that 
$d\sum_{i=1}^m p_i u_i u_i^\top = I_d$. Then, there exists a (possibly multi) set $S \subseteq \{u_1,\ldots,u_m\}$ such that $\sum_{v \in S} vv^\top \succeq I_d$ and $|S| \leq 12d$. Moreover, such a set can be computed in time $O(md^3)$.
\end{lemma}

\begin{proof} [Proof of Lemma~\ref{lem:pruning contact}]
We first state Theorem 3.1 of \cite{batson2012twice}.
\begin{theorem}
Let $v_1,\ldots,v_m$ be vectors in $\R^d$ and let $c>1$. Assume that $\sum_i v_i v_i^\top = I_d$. There exist scalars $s_i \geq 0$ for $i \in [m]$ with $|\{i| s_i>0\}| \leq cd$ such that 
$$ I_d \preceq \sum_i s_i v_i v_i^\top \preceq \frac{c+1+2\sqrt{c}}{c+1-2\sqrt{c}}I_d $$
Furthermore, these scalars can be found in time $O(cd^3 m)$.
\end{theorem}

Let $v_i = \sqrt{p_i}u_i$. We fix $c$ as some constant whose value will be determined later. Clearly for these vectors $v_i$ it holds that $d \sum_i v_i v_i^\top = I_d$. It follows from the above theorem that there exist some scalars $s_i \geq 0$ for which at most $cd$ are non-zeros and
\begin{equation}\label{eq:si vi}
I_d \preceq d \sum_i s_i v_i v_i^\top \preceq \frac{c+1+2\sqrt{c}}{c+1-2\sqrt{c}}I_d
\end{equation}
Our set $S$ will be composed by taking each $u_i$ $\ceil{d s_i p_i}$ many times. Plugging in equation~\eqref{eq:si vi} shows that indeed
$$ \sum_{w \in S} w w^\top \succeq I_d $$
and it remains to bound the size of $S$ i.e., $\sum \ceil{d s_i p_i}$.
By taking the trace of the expression and dividing by $d$ we get that
$$ \sum_i s_i \mathrm{Trace}(v_i v_i^\top) \preceq \frac{c+1+2\sqrt{c}}{c+1-2\sqrt{c}} $$
Plugging in the expressions for $v_i$ along with $u_i$ being unit vectors (hence $\mathrm{Trace}(u_iu_i^\top)=1$) lead to 
$$ \sum_i s_i p_i \preceq \frac{c+1+2\sqrt{c}}{c+1-2\sqrt{c}} $$
It follows that 
$$ \sum \ceil{d s_i p_i} \leq d \frac{c+1+2\sqrt{c}}{c+1-2\sqrt{c}} + |\{i|s_i \geq 0\}| \leq d\left(c+\frac{c+1+2\sqrt{c}}{c+1-2\sqrt{c}} \right)  $$
By optimizing $c$ we get $ \sum \ceil{d s_i p_i} \leq 12d$, and the lemma is proved.
\end{proof}

Lemma~\ref{lem:pruning contact} proves the existence of a spanner of size at most $12d$ for any compact set in $\R^d$. Furthermore, the running time to find the set is $O(md^3)=O(d^5)$, using the bound of $m=O(d^2)$. This along with the running time for obtaining the contact points leads to a total running time of $O(n^{3.5}+dn^3+d^5)$. 
\end{proof}

\section{Approximate Volumetric Spanners} \label{sec:algcont}

In this section we provide a construction for $(p,\eps)$-\probspanner (as in Definition~\ref{def:volume apx}), proving Theorem~\ref{thm:vol apx}.
We start by providing a more technical definition of a spanner. Note that unlike previous definitions, the following is not impervious to linear operators and will only be used to aid our construction.

\begin{definition}
A  $\beta$-relative-spanner is a discrete subset $S \subseteq \K$ such that for all $x \in \K$, $\|x\|_{\mE(S)}^2 \leq \beta \|x\|^2$.
\end{definition}

A first step is a spectral characterization of relative spanners:
\begin{lemma} \label{lem:relative span}
Let $S = \{v_1,...,v_T\} \subseteq \K$ span $\K$ and be such that 
$$ W = \sum_{i=1}^T v_i v_i^\top \succeq \frac{1}{\beta} I_d $$
Then $S$ is a $\beta$-relative-spanner. 
\end{lemma}
\begin{proof}
Let $V \in \reals^{d \times T}$ be a matrix whose columns are the vectors of $S$. As $V V^\top = W \succeq \frac{1}{\beta}I_d$ we have that 
$$ \beta I_d \succeq (VV^\top)^{-1} $$
It follows that
$$ \|x\|_{\mE(S)} = x^\top (VV^\top)^{-1} x \leq \beta \|x\|^2 $$
as required.
\end{proof}


\begin{algorithm}[h!]
\caption{  }
    \begin{algorithmic}[1]
    \STATE Input: An oracle to $x \sim p$, where $p$ is a distribution over $\K$, $T \in \mathbb{N}$.
    	\FOR { $t=1$ to $T$} 
            \STATE Choose $u_t \sim p$ 
            \ENDFOR
            \RETURN $S = \{u_1,...,u_T\} $. 
    \end{algorithmic}
   \label{alg:generic}
\end{algorithm}

\begin{proof} [Proof of Theorem~\ref{thm:vol apx}]
We analyze Algorithm~\ref{alg:generic}, previously defined within Theorem~\ref{thm:vol apx} assuming the vectors are sampled exactly according to the log-concave distribution. The result involving an approximate sample, which is necessary for implementing the algorithm in the general case, is an immediate application of Lemma~\ref{lem:log-conv-sample} and Corollary~\ref{cor:logconcave}.

Our analysis of the algorithm is for $T = C(d+\log^2(1/\eps))$ samples, where $C$ is some sufficiently large constant.
Assume first that $\E_{x \sim p}[xx^\top] = I_d$. 
Let $W = \sum_{i=1}^T u_i u_i^\top$. Then, for $C > 0$ large enough, by Corollary~\ref{cor:logconcave}, $\|\frac{1}{T}W - I_d\| \leq 1/2$ w.p.\ at least $1-\exp(-\sqrt{d})$.  Therefore, $S$ spans $\reals^d$ and 
$$  \frac{1}{T}W \succeq I_d - \frac{1}{2}I_d = \frac{1}{2}I_d$$
Thus according to Lemma~\ref{lem:relative span}, $S$ is a $(2/T)$-relative spanner. 
Consider the case where $\Sigma =\E_{x \sim p}[xx^\top]$ is not necessarily the identity. By the above analysis we get that $\Sigma^{-1/2}S=\{\Sigma^{-1/2}u_1,\ldots,\Sigma^{-1/2}u_T\}$ form a $(2/T)$-relative spanner for $\Sigma^{-1/2} \K$. This is since the r.v defined as $\Sigma^{-1/2}x$ where $x \sim p$ is log-concave.
The latter along with corollary~\ref{cor:quasi_iso} implies that for any $\theta \geq 1$,
\begin{equation} \label{eq:sigmax 3 theta}
\Pr_{x \sim p} \left[ \|\Sigma^{-1/2} x\| \geq 3\theta \sqrt{d} \right] \leq c_1\exp\left(-c_2\theta \sqrt{d}\right) 
\end{equation}
for some universal constants $c_1,c_2>0$. It follows that for our set $S$ and any $\theta \geq 1$,

\begin{eqnarray*}
\Pr_{x \sim p}\left[ \|x\|_{\mE(S)} > \theta \right] & = \Pr_{x \sim p}\left[ \|\Sigma^{-1/2}x\|_{\mE(\Sigma^{-1/2}S)} > \theta\right] & \mbox{$\|x\|_{\mE(S)} =  \|\Sigma^{-1/2}x\|_{\mE(\Sigma^{-1/2}S)}$} \\ 
& \leq  \Pr_{x \sim p}\left[ \|\Sigma^{-1/2}x\| > \theta \sqrt{T/2}\right] & \mbox{$\Sigma^{-1/2}S$ is a $2/T$-relative-spanner} \\
& = \Pr_{x \sim p}\left[ \|\Sigma^{-1/2}x\| > 3 \theta \sqrt{d}  \sqrt{\frac{C}{18}} \cdot \sqrt{1+\frac{\log^2(1/\eps)}{d}}\right]  & \mbox{$T = C(d+\log^2(1/\eps))$} \\
& \leq c_1 \exp\left( -c_2\theta \sqrt{d}  \sqrt{\frac{C}{18}} \cdot \sqrt{1+\frac{\log^2(1/\eps)}{d}} \right) &  \mbox{Equation~\eqref{eq:sigmax 3 theta}, $C\geq 18$} \\
& \leq \exp\left( -\theta  \sqrt{d+\log^2(1/\eps)} \right) &  \mbox{$C$ sufficiently large}  \\
& \leq \eps^{-\theta} &  
\end{eqnarray*}

\end{proof}


In our application of volumetric spanners to BLO, we also need the following relaxation of volumetric spanners where we allow ourselves the flexibility to scale the ellipsoid:
\begin{definition} \label{def:rho apx}
A $\rho$-\ratiospanner $S$ of $\K$ is a subset $S \subseteq \K$ such that for all $x \in \K$, $\|x\|_{\mE(S)} \leq \rho.$
\end{definition}

One example for such an approximate spanner with $\rho = \sqrt{d}$ is a barycentric spanner (Definition~\ref{def:barycentric}). In fact, it is easy to see that a $C$-approximate barycentric spanner is a $C\sqrt{d}$-\ratiospanner. The following is immediate from Theorem~\ref{thm:barycentric}.
\begin{corollary}  \label{cor:ratio sqrt d}
Let $\K$ be a compact set in $\R^d$ that is not contained in any proper linear subspace. Given an oracle for optimizing linear functions over $\K$, for any $C>1$, it is possible to compute a $C\sqrt{d}$-\ratiospanner $S$ of $\K$ of cardinality $|S|=d$, using $O(d^2\log_C(d))$ calls to the optimization oracle.
\end{corollary}
\section{Fast Volumetric Spanners for Discrete Sets}\label{sec:algdiscrete}
In this section we describe a different algorithm that constructs volumetric spanners for discrete sets. The order of the spanners we construct here is suboptimal (in particular, there is a dependence on the size of the set $\K$ which we didn't have before). However, the algorithm is particularly simple and efficient to implement (takes time linear in the size of the set).

\begin{algorithm}[h!]
\caption{  }
    \begin{algorithmic}[1]
    \STATE Input $\K = \{x_1,...,x_n\} \subseteq \reals^d$.
    \IF {$n < C d\log d$}  
     \RETURN $S \leftarrow \K$
     \ENDIF
\STATE Compute $\Sigma = \sum_i x_i x_i^\top$ and let $u_i = \Sigma^{-1/2} x_i$. 
\STATE For $i \in [n]$, let $p_i = 1/2n + \|u_i\|^2/2d$. Let $S$ be a random set obtained by drawing $M = C d\log d$ samples with replacement from $[n]$ according to the distribution $p_1,\ldots,p_n$. \label{stp:S from iso}
\STATE Verify that for at least $n/2$ vectors from $\{x_1,\ldots,x_n\}$, it holds that $\|x_i\|_{\mE(S)} \leq 1$. If that is not the case discard $S$ and repeat the above step.
\STATE Apply the algorithm recursively on the data points for which $\|x_i\|_{\mE(S)} > 1$.
    \end{algorithmic}
   \label{alg:discrete}
   \end{algorithm}

\begin{theorem} \label{thm:light const}
Given a set of vectors $\K = \{x_1,\ldots,x_n\} \in \reals^d$, Algorithm \ref{alg:discrete} outputs a volumetric spanner of size $O((d\log d)(\log n))$ and has an expected running time of $O(n d^2)$. 
\end{theorem}

\begin{proof} 
Consider a single iteration of the algorithm with input $v_1,\ldots,v_n \in \reals^d$. We claim that the random set $S$ obtained in step \ref{stp:S from iso} satisfies the following condition with constant probability: 
\begin{equation} \label{eq:disc1}
\Pr_{x \in \K}\left[ \|x\|_{\mE(S)} \leq 1 \right] \geq 1/2
\end{equation}

Suppose the above statement is true. Then, the lemma follows easily as it implies that for the next iteration there are fewer than $n/2$ vectors. Hence, after $(\log n)$ recursive calls we will have a volumetric spanner. The total size of the set will be $O((d\log d)(\log n))$. To see the time complexity, consider a single run of the algorithm. The most computationally intensive steps are computing $\Sigma$ and $\Sigma^{-1/2}$ which take time $O(nd^2)$ and $O(d^3)$ respectively. We also need to compute $(\sum_{v\in S} v v^\top)^{-1}$ (to compute the $\mE(S)$ norm) which takes time $O(d^3 \log d)$, and compute the $\mE(S)$ norm of all the vectors which requires $O(nd^2)$. As $n = \Omega(d\log(d))$, it follows that a single iteration runs of a total expected time of $O(nd^2)$. Since the size of $n$ is split in half between iterations, the claim follows.

We now prove that Equation \ref{eq:disc1} holds with constant probability
\begin{equation}
  \label{eq:disc2}
 x_j^\top \left(\sum_{v \in S} v v^\top\right)^{-1} x_j = u_j^\top \left(\sum_{v \in S'} v v^\top\right)^{-1} u_j.  
\end{equation}
where $S' = \{ \Sigma^{-1/2}v | v \in S \}$ is the (linearly) shifted version of $S$.
Therefore, it suffices to show that with sufficiently high probability, the right hand side of the above equation is bounded by $1$ for at least $n/2$ indices $j \in [n]$.

Note that $p_i = 1/2n + \|u_i\|^2/2d$ form a probability distribution: $\sum_i p_i = 1/2 + (\sum_i \|u_i\|^2)/2d = 1$. Let $X \in \reals^d$ be a random variable with $X = u_i/\sqrt{p_i}$ with probability $p_i$ for $i \in [n]$. Then, $\E[XX^\top] = I_d$. Further, for any $i \in [n]$
$$\|u_i\|^2/p_i \leq 2d.$$
Therefore, by Theorem \ref{th:rudelson}, if we take $M = C d (\log d)$ samples $X_1,\ldots,X_M$ for $C$ sufficiently large, then with probability of at least $1/2$, it holds that 
$$\sum_{i=1}^M X_i X_i^\top \succeq (M/2) I_d.$$
Let $T \subseteq [n]$ be the multiset corresponding to the indices of the sampled vectors $X_1,\ldots,X_M$. The above inequality implies that
$$\sum_{i \in T} \frac{1}{p_i} u_i u_i^\top \succeq (M/2) I_d.$$
Now, 
$$\sum_{v \in S'} v v^\top \succeq (\min_i p_i) \sum_{v \in S'} \frac{1}{p_i} v v^\top \succeq (\min_i p_i) (M/2) I_d \succeq (M/4n) I_d.$$
Therefore,
\begin{align*}
 \sum_{i=1}^n u_i^\top \left(\sum_{v \in S'} v v^\top\right)^{-1} u_i &= \sum_{i=1}^n \mathrm{Tr} \left(\left(\sum_{v \in S'} v v^\top\right)^{-1}\left(u_i u_i^\top\right)\right)\\
&= \mathrm{Tr} \left(\left(\sum_{v \in S'} v v^\top\right)^{-1}\left(\sum_{i=1}^n u_i u_i^\top\right)\right)\\
&= \mathrm{Tr} \left(\left(\sum_{v \in S'} v v^\top\right)^{-1}\right) \\
&\leq \frac{4nd}{M} \leq \frac{4n}{C\log d} \leq \frac{n}{2\log d},
\end{align*}
for $C$ sufficiently large. Therefore, by Markov's inequality and Equation \ref{eq:disc2}, it follows that Equation \ref{eq:disc1} holds with high probability. The theorem now follows.
\end{proof}

\section{Bandit Linear Optimization} \label{sec:blo}

Recall the  problem of Bandit Linear Optimization (BLO): iteratively  at each time sequence $t$, the environment chooses a loss vector $L_t$ that is not revealed to the player. The player chooses a vector $x_t \in \K$ where $\K \subseteq \reals^d$ is convex, and once she commits to her choice, the loss $\ell_t = x_t^\top L_t$ is revealed. The objective is to minimize the loss and specifically, the regret, defined as the strategy's loss minus the loss of the best fixed strategy of choosing some $x^* \in \K$ for all $t$. We henceforth  assume
that the loss vectors $L_t$'s are chosen from the polar of $\K$, meaning from $\{L: |L^\top x| \leq 1 \ \forall x \in \K\}$. In particular this means that the losses are bounded in absolute value, although a different choice of assumption (i.e. $\ell_\infty$ bound on the losses) can yield different regret bounds, see discussion in \cite{DBLP:journals/jmlr/AudibertBL11}.

The problem of BLO is a natural generalization of the classical Multi-Armed Bandit problem and extremely useful for efficiently modeling decision making under partial feedback for structured problems. As such the research literature is rich with algorithms and insights into this fundamental problem. For a brief historical survey please refer to earlier sections of this manuscript. In this section we focus on the first efficient and optimal-regret algorithm, and thus immediately jump to Algorithm \ref{alg:BLO}. We make the following assumptions over the decision set $\K$:
\begin{enumerate}
\item
The set $\K$ is equipped with a membership oracle. This implies via \cite{lovasz2007geometry} (Lemma~\ref{lem:log-conv-sample}) that there exists an efficient algorithm for sampling from a given log-concave distribution over $\K$. Via the discussion in previous sections, this also implies that we can construct approximate (both types of approximations, see Definitions~\ref{def:rho apx} and \ref{def:volume apx}) volumetric spanners efficiently over $\K$. 
\item
The losses are bounded in absolute values by 1. That is, the loss functions are always chosen (by an oblivious adversary)  from a convex set $\cal Z$ such that $\K$ is contained in its polar, i.e. $\forall L \in {\cal Z}, x \in \K$, $|L^\top x| \leq 1$.
This implies that the set $\K$ admits for any $\eps > 0$ an $\eps$-net, w.r.t the norm defined by $\cal Z$, whose size we denote by $|K|_\eps \leq (\eps/2)^{-d}$.

\end{enumerate}
For  Algorithm \ref{alg:BLO} we prove the following optimal regret bound:

\begin{algorithm}[h!]
\caption{ GeometricHedge with Volumetric Spanners Exploration } 
    \begin{algorithmic}[1]
    \STATE $\K$, parameters $\gamma, \eta$, horizon $T$.
		\STATE $p_1(x)$ uniform distribution over $\K$.
    \FOR { $t=1$ to $T$} 
						\STATE Let $S_t'$ be a $(p_t,\exp(-(4\sqrt{d}+ \log(2T))))$-\probspanner of $\K$. \label{alg:vspannerstep}
						\STATE Let $S_t''$ be a $2\sqrt{d}$-\ratiospanner of $\K$
						\STATE Set $S_t$ as the union of $S_t',S_t''$.
            \STATE $\hat{p}_t(x) = (1-\gamma)p_t(x) + \frac{\gamma}{|S_t|} 1_{x \in S_t} $
						\STATE sample $x_t$ according to $\hat{p}_t$
						\STATE observe loss $\ell_t \eqdef L_t^\top x_t $
						\STATE Let $C_t \eqdef \E_{x \sim \hat{p}_t} [x x^\top]$
						\STATE $\hat{L}_t \eqdef \ell_t C_t^{-1}x_t$
						\STATE $p_{t+1}(x) \propto p_t(x) e^{-\eta \hat{L}_t^\top x}$
    \ENDFOR
         
    \end{algorithmic}
   \label{alg:BLO}
\end{algorithm}

\begin{theorem} \label{thm:geo hedge}
Under the assumptions stated above, and let $s=\max_t |S_t|$, $ \eta = \sqrt{\frac{\log |\K|_{1/T}}{dT}}$ and let $\gamma = s\sqrt{\frac{\log(|\K|_{1/T})}{dT}} $. Algorithm \ref{alg:BLO} given parameters $\gamma,\eta$ suffers a regret bounded by 
$$ O\left(  (s+d)\sqrt{\frac{T \log |\K|_{1/T} }{d}} \right)$$
\end{theorem}

We note that while the size $\log(|\K|_{1/T})$ can be bounded by  $d \log(T)$, in certain scenarios such as s-t paths in graphs it is possible to obtain sharper upper bounds that immediately imply better regret  via  Theorem \ref{thm:geo hedge}. 

\begin{corollary}
There exist an efficient algorithm for BLO for any convex set $\K$ with regret of 
$$ O\left(  \sqrt{dT \log |\K|_{1/T}  } \right) = O\left(  d\sqrt{T \log(T)  } \right) $$
\end{corollary}
\begin{proof}
The spanner in step \ref{alg:vspannerstep} of the algorithm does not have to be explicitly constructed. According to  Theorem~\ref{thm:vol apx}, to obtain such as spanner it suffices to  sample sufficiently many points from the distribution $p_t$, hence this portion of the exploration strategy is identical to the exploitation strategy. 

According to Corollary~\ref{cor:ratio sqrt d}, a $2\sqrt{d}$-\ratiospanner of size $d$ can be efficiently constructed, given a linear optimization oracle which in turn can be efficiently implemented by the membership oracle for $\K$. Hence, it follows that for the purpose of the analysis, $s=d$ and the bound follows.
\end{proof}

To prove the theorem we follow the general methodology used in analyzing the performance of the geometric hedge algorithm. The major deviation from standard technique is the following sub-exponential tail bound, which we use to replace the the standard  absolute bound for $|\hat{L}_t x|$. After giving its proof and a few auxiliary lemmas, we give the proof of the main theorem.

\begin{lemma} \label{lem:Ltx bounded}
Let $x \sim p_t$, $x_t \sim \hat{p}_t$ and let $\hat{L}_t$ be defined according to $x_t$. It holds, for any $\theta > 1$ that
$$ \Pr\left[ |\hat{L}_t^\top x| > \frac{\theta s}{\gamma} \right] \leq \exp(-2\theta)/T $$ 
\end{lemma}
\begin{proof}
\begin{eqnarray*}
\Pr\left[ |\hat{L}_t^\top x| > \frac{\theta s}{\gamma} \right] & \leq \Pr\left[ \|x\|_{\mE(S_t)} \cdot \|x_t\|_{\mE(S_t)} \geq \theta \right] & \mbox{ Lemma \ref{lem:Ltx_mES} } \\
& \leq \Pr\left[ \|x\|_{\mE(S_t)} \geq \sqrt{\theta} \ \ \bigvee \ \  \|x_t\|_{\mE(S_t)} \geq \sqrt{\theta} \right] &   \\
& \leq \Pr\left[ \|x\|_{\mE(S_t)} \geq \sqrt{\theta} \right] + \Pr\left[ \|x_t\|_{\mE(S_t)} \geq \sqrt{\theta} \right] &   \\
& \leq 2\Pr\left[ \|x\|_{\mE(S_t)} \geq \sqrt{\theta} \right] & 
\end{eqnarray*}
To justify the last inequality notice that $x \sim p_t$ and $x_t \sim \hat{p}_t$ where $\hat{p}_t$ is a convex sum of $p_t$ and a distribution $q_t$ for which $\Pr_{y \sim q_t}\left[ \|y\|_{\mE(S_t)} \geq \sqrt{\theta} >1 \right]=0$. Before we continue recall that we can assume that $\sqrt{\theta} \leq 2\sqrt{d}$, since $S_t''$ is a $2\sqrt{d}$-\ratiospanner. 
\begin{eqnarray*}
\Pr\left[ |\hat{L}_t^\top x| > \frac{\theta s}{\gamma} \right] & \leq 2\Pr\left[ \|x\|_{\mE(S_t)} \geq \sqrt{\theta} \right] & \\
& \leq 2\exp(- \sqrt{\theta} ( 4\sqrt{d} + \log 2T)  )  & \mbox { property of \probspanner} \\
& \leq \frac{1}{T} \exp(- 2 \sqrt{\theta \cdot 4d}   ) \\
& \leq \frac{1}{T} \exp(- 2\theta  ) & \mbox{ since $\theta \leq 4d$}
\end{eqnarray*}
\end{proof}

\begin{lemma} \label{lem:Ltx_mES}
For all $x \in \K$ it holds that $|\hat{L}_t^\top x| \leq \frac{|S_t|\|x\|_{\mE(S_t)}\|x_t\|_{\mE(S_t)}}{\gamma}$.
\end{lemma}
\begin{proof}
Let $x \in \K$.  Denote by $V_t$ the matrix whose columns are the elements of $S_t$ and recall that $\|y\|_{\mE(S_t)}^2 = y^\top (V_t V_t^\top)^{-1} y$.
Since $C_t \eqdef \E_{x \sim \hat{p}_t} [x x^\top]$, it holds that 
$$C_t \succeq  \frac{\gamma}{|S_t|}\sum_{v \in S_t} vv^\top = \frac{\gamma}{|S_t|} V_tV_t^\top$$
since both matrices are full rank, it holds that 
$$C_t^{-1} \preceq \frac{|S_t|}{\gamma}(V_tV_t^\top)^{-1}$$
Notice that due to the Cauchy-Schwartz inequality,
$$ |x^\top \hat{L}_t| = |\ell_t| \cdot |x^\top C_t^{-1} x_t| \leq |\ell_t| \cdot\|x^\top C_t^{-1/2}\| \cdot \|C_t^{-1/2} x_t\|$$
The matrix $C_t^{-1/2}$ is defined as $C_t$ is positive definite. Now, 
$$\|x^\top C_t^{-1/2} \|^2 = x^\top C_t^{-1} x \leq x^\top \frac{|S_t|}{\gamma} (V_tV_t^\top)^{-1} x = \frac{|S_t|}{\gamma} \|x\|_{\mE(S_t)}^2 $$
Since the analog can be said for $\|C_t^{-1/2} x_t\|$ (as $x_t \in \K$), it follows that 
$$|x^\top \hat{L}_t| \leq |\ell_t| \frac{|S_t| \|x\|_{\mE(S_t)} \|x_t\|_{\mE(S_t)} }{\gamma} \leq \frac{|S_t| \|x\|_{\mE(S_t)} \|x_t\|_{\mE(S_t)} }{\gamma}$$
The last inequality is since we assume the rewards are in $[-1,1]$.
\end{proof}

\paragraph{Implementation for general convex bodies. } In the case where the set $\K$ is a general convex body, the  analysis must include the fact that we can only approximately sample a log-concave distribution over $\K$. As the main focus of our work is to prove a polynomial solution we present only a simple analysis yielding a running time polynomial in the dimension $d$ and horizon $T$. It is likely that a more thorough analysis can substantially reduce the running time.
\begin{corollary}
In the general case where an approximate sampling is required, algorithm~\ref{alg:BLO} can be implemented with a running time of $\tilde{O}(d^5 + d^3 T^6)$ per iteration.
\end{corollary}
\begin{proof}
%
%
%

\newcommand{\dstat}[1]{\|#1\|_{TV}}
Fix an error parameter $\delta$, and let us run Algorithm \ref{alg:BLO} with the approximate samplers $p_t'$ guaranteed by Theorem \ref{lem:log-conv-sample}. Then, in each use of the sampler we are replacing the true distribution we should be using $p_t$, with a distribution $p_t'$ such that statistical distance between $p_t, p_t'$ is at most $\delta$. Let us now analyze the error incurred by this approximation by bounding the loss from the first round onwards. Suppose the algorithm ran with the approximate sampler in the first round but the exact sampler in each round afterwards. Then, as the statistical distance between the distributions is at most $\delta$ and the loss in each round is bounded by $1$ and there are $T$ rounds, the net difference in expected regret between using $p_1$ and $p_1'$ will be at most $\delta \cdot T$. Similarly, if we ran the algorithm with $p_1',\ldots,p_{i-1}',p_i',p_{i+1},\cdots, p_T$ as opposed to $p_1',\ldots,p_{i-1}',p_i,p_{i+1},\cdots,p_T$\footnote{Here, $p_j$'s are interpreted as the distribution given by the algorithm based on the distribution from previous round and $p_j'$ is the approximate oracle for this distribution $p_j$.} (we are changing the $i$'th distribution from exact to approximate), the net difference in expected regret would be at most $\delta \cdot T$. Therefore, the total additional loss we may incur for using the approximate oracles is at most $T \cdot (\delta T) = \delta T^2$. Thus, if we take $\delta = \Delta/T^2$, where $\Delta$ is the regret bound from Theorem \ref{thm:geo hedge}, we get a regret bound of $2\Delta$. The required value of $\delta$ is bounded by $T^{-1.5}$. Applying Theorem \ref{lem:log-conv-sample} leads to a running time of $\tilde{O}(d^5 + d^3 T^6)$ per iteration.
\end{proof}

\subsection{Proof of Theorem~\ref{thm:geo hedge}} \label{app:BLO}

We continue the analysis of the Geometric Hedge algorithm similarly to \cite{dani2007price,BubeckCK12}, under certain assumptions over the exploration strategy. For convenience we will assume that the set of possible arms $\K$ is finite. This assumption holds w.l.o.g since if $\K$ is infinite, a $\sqrt{1/T}$-net of it can be considered as described earlier (this will have no effect on the computational complexity of our algorithm, but a mere technical convenience in the proof below).

Before proving the theorem we will require three technical lemmas. In the first we show that  $\hat{L}_t$ is an unbiased estimator of $L_t$. In the second, we bound a proxy of its variance. In the third, we bound a proxy of the expected value of its exponent.
\begin{lemma} \label{Lt unbiased}
In each $t$, $\hat{L}_t$ is an unbiased estimator of $L_t$
\end{lemma}
\begin{proof}
$$ \hat{L}_t = \ell_t C_t^{-1} x_t = (L_t^\top x_t) C_t^{-1} x_t = C_t^{-1} (x_t x_t^\top) L_t $$
Hence,
$$ \E_{x_t \sim p_t} [\hat{L}_t] = C_t^{-1} \E_{x_t \sim p_t} [x_t x_t^\top] L_t = C_t^{-1} C_t L_t =L_t$$
\end{proof}

\begin{lemma} \label{lem:xCx bound} 
Let $t \in [T]$, $x \sim p_t$ and $x_t \sim \hat{p}_t$. It holds that $\E[(\hat{L}_t^\top x)^2 ] \leq d/(1-\gamma) \leq 2d $
\end{lemma}
\begin{proof}
For convenience, denote by $q_t$ the  uniform distribution over $S_t$ -  the exploration strategy at round $t$. First notice that for any $x \in \K$, 
\begin{eqnarray} \label{eq:xCx}
\E_{x_t \sim \hat{p}_t} [(\hat{L}_t^\top x)^2 ] & = x^\top \E_{x_t \sim \hat{p}_t}  [\hat{L}_t \hat{L}_t^\top ] x = x^\top \E_{x_t \sim \hat{p}_t}  [\ell_t^2 C_t^{-1} x_t x_t^\top C_t^{-1} ] x \notag \\ 
& =  \ell_t^2  x^\top C_t^{-1} \E_{x_t \sim \hat{p}_t}  [ x_t x_t^\top  ] C_t^{-1} x =  \ell_t^2 x^\top C_t^{-1} x \notag \\
& \leq x^\top C_t^{-1} x 
\end{eqnarray}
Next,
$$ \E_{x \sim \hat{p}_t} [x^\top C_t^{-1} x ] = \E_{x \sim \hat{p}_t} [ C_t^{-1} \bullet xx^\top ] = C_t^{-1} \bullet \E_{x \sim \hat{p}_t} [xx^\top] = C_t^{-1} \bullet C_t  = \trace(I_d) =  d$$
Where we used linearity of expectation and denote $A \bullet B = \trace(AB)$. Since $C_t^{-1}$ is positive semi definite,
\begin{equation} \label{eq:xCx2}
 (1-\gamma)\E_{x \sim p_t} [x^\top C_t^{-1} x ] \leq (1-\gamma)\E_{x \sim p_t} [x^\top C_t^{-1} x ] + \gamma \E_{x \sim q_t} [x^\top C_t^{-1} x ] = \E_{x \sim \hat{p}_t} [x^\top C_t^{-1} x ] = d
\end{equation}
The lemma follows from combining Equations~\ref{eq:xCx} and~\ref{eq:xCx2}.
\end{proof}

\begin{lemma} \label{lem:e exp theta}
Denote by $\bone_{\phi}$ the random variable taking a value of $1$ if event $\phi$ occurred and 0 otherwise.
Let $t\in [T]$, $x_t \sim \hat{p}_t$ and $x \sim p_t$. For $\hat{L}_t$ defined by $x_t$ it holds that
$$\E\left[ \exp( -\eta \hat{L}_t^\top x )  \bone_{ -\eta \hat{L}_t^\top x > 1 } \right] \leq \frac{2}{T}$$
\end{lemma}
\begin{proof}
Let $f,F$ be the pdf and cdf of the random variable $Y = -\eta \hat{L}_t^\top x$ correspondingly. From Lemma ~\ref{lem:Ltx bounded} and the fact that $1/\eta=s/\gamma$ ($s=\max_t |S_t|$) we have that for any $\theta \geq 1$, 
$$1-F(\theta) \leq \frac{1}{T} e^{-2\theta}$$
and we'd like to prove that under this condition, 
$$\E[ e^{Y} \bone_{Y > 1} ] = \int_{\theta = 1}^\infty e^{\theta} f(\theta) d \theta \leq \frac{2}{T}$$
which follows from the definition of the cdf and pdf:
\begin{eqnarray*}
\E[ e^{Y} \bone_{Y > 1} ] & = \int_{\theta = 1}^\infty e^{\theta} f(\theta) d \theta \\
& = \sum_{k=1}^\infty  \int_{\theta = k}^{k+1} e^{\theta} f(\theta)  d \theta \\
& \leq \sum_{k=1}^\infty e^{k+1} \int_{\theta = k}^{k+1}  f(\theta)  d \theta \\
& \leq \sum_{k=1}^\infty e^{k+1} (F(k+1) - F(k) )\\
& \leq \sum_{k=1}^\infty e^{k+1} (1 - F(k) ) \\
& \leq \sum_{k=1}^\infty e^{k+1} \cdot \frac{1}{T} e^{-2k} & \mbox{ Lemma ~\ref{lem:Ltx bounded} } \\
& = \frac{e}{T} \sum_{k=1}^\infty e^{-k}  = \frac{e}{T} \cdot \frac{e^{-1} }{ 1 - e^{-1} }  \leq \frac{2}{T} 
\end{eqnarray*}

\end{proof}

\begin{proof} [Proof of Theorem~\ref{thm:geo hedge}]
For convenience we define within this proof for $x\in \K$, ${\lhat}_{1:t-1}(x) \eqdef \sum_{i=1}^{t-1} \hat{L}_i^\top x$ and let $\lhat_t(x) \eqdef \hat{L}_t^\top x$. 
Let $W_t = \sum_{x \in \K} \exp(-\eta {\lhat}_{1:t-1}(x))$. For all $t \in [T]$: 
\begin{eqnarray*}
\E\left[\frac{W_{t+1}}{W_t} \right] & = \E \left[ \sum_{x \in \K} \frac{\exp(-\eta {\lhat}_{1:t-1}(x)) \exp(-\eta \lhat_t(x)) }{W_t} \right]  \\
& = \E_{x_t \sim \hat{p}_t} \left[\sum_{x \in \K} p_t(x) \exp(-\eta \lhat_t(x)) \right] \\
& = \E_{x_t \sim \hat{p}_t, x \sim p_t} [\exp(-\eta \lhat_t(x))] \leq \\
& \leq 1 - \eta \E[\hat{L}_t^\top x] + \eta^2 \E[(\hat{L}_t^\top x)^2] + \E\left[ \exp( -\eta \hat{L}_t^\top x )  \bone_{ -\eta \hat{L}_t^\top x > 1 } \right] \\
& \\
& \mbox{ using the  inequality $\exp(y) \leq 1+y+y^2 + \exp(y) \cdot {\bf 1}_{y>1} $ } \\
& \\
& \leq 1 - \eta \E[\hat{L}_t^\top x] + \eta^2 \E[(\hat{L}_t^\top x)^2] + \frac{2}{T}   & \mbox{ Lemma ~\ref{lem:e exp theta}}
\end{eqnarray*} 

Since $\hat{L}_t$ is an unbiased estimator of $L_t$ (Lemma~\ref{Lt unbiased}) and according to Lemma~\ref{lem:xCx bound}, $\E[(\hat{L}_t^\top x)^2] \leq 2d$, we get:
\begin{equation} \label{eqn:shalom}
\E\left[\frac{W_{t+1}}{W_t} \right] \leq 1 - \eta L_t^\top \E_{x \sim p_t}[x] + 2 \eta^2 d + \frac{2}{T} 
\end{equation}
We now use Jensen's inequality:
\begin{eqnarray*}
\E[\log(W_T)] - \E[\log(W_1)] & = \E[\log(W_T/W_1)]\\
& =  \sum_{t=1}^{T-1} \E[\log(W_{t+1}/W_t)] \\
& \leq \sum_{t=1}^{T-1} \log(\E[W_{t+1}/W_t]) & \mbox{ Jensen} \\
& \leq  \sum_{t=1}^{T-1} \log \left( 1 - \eta L_t^\top \E_{x \sim p_t}[x] + 2\eta^2 d + \frac{2}{T} \right)  & \mbox{ \eqref{eqn:shalom} }\\
& \leq  \sum_{t=1}^{T-1} - \eta L_t^\top \E_{x \sim p_t}[x]  + 2\eta^2 d + \frac{2}{T}  & \mbox{ Due to $\ln(1+y) \leq y$ for all $y > -1$} \\
& \leq 2+ 2\eta^2 Td  - \eta \sum_t \E_{x \sim p_t}[L_t^\top x] 
\end{eqnarray*}

Now, since $\log(W_1) = \log(|\K|)$ and $W_T \geq \exp(-\eta{\lhat}_{1:T} (x^*))$ for any $x^* \in \K$, by shifting sides of the above it holds for any $x^* \in \K$ that
$$ \sum_t \E_{x \sim p_t}[L_t^\top x] - \sum_t L_t^\top x^*  \leq \sum_t \E_{x \sim p_t}[L_t^\top x] + \E[ \log W_T]  \leq \frac{\log(|\K|)+2}{\eta} + 2\eta T d $$
Finally, by noticing that 
$$ \sum_t \E_{x \sim \hat{p}_t}[L_t^\top x] - \sum_t \E_{x \sim p_t}[L_t^\top x] \leq \gamma T$$
we obtain a bound of 
$$ \E[ \mbox{Regret}] = \E[  \sum_t L_t^\top x_t] - \sum_t L_t^\top x^* =  \sum_t \E_{x \sim \hat{p}_t}[L_t^\top x] - \mathrm{Loss}(x^*) \leq  \frac{\log(|\K|)+2}{\eta} + 2\eta T d + \gamma T $$
on the expected regret. By plugging in the values of $\eta,\gamma$
we get the bound of
$$ O\left( (s+d)\sqrt{\frac{T \log(|\K|)}{d}} \right)$$
as required.
\end{proof}

\paragraph{Discussion: } Notice that to obtain a $(p,\eps)$-\probspanner for a log-concave distribution $p$ over a body $\K$ we simply choose sufficiently many i.i.d samples from $p$. Since in the above algorithm $p_t$ is always log-concave, it follows that $S_t'$ consists of i.i.d samples from $p_t$, meaning that if we would not have required $S_t''$, the exploration and exploration strategies would be the same! Since we still require the set $S_t''$, there exists a need for a separate exploration strategy. 
Interestingly, the $2\sqrt{d}$-\ratiospanner is obtained by taking a barycentric spanner, which is the exploration strategy of \cite{dani2007price}.

\bibliographystyle{plain}
\bibliography{vspanner}

\appendix

\section{Concentration bounds for non centered isotropic log concave distributions} \label{app:logcon concentrate}

We begin by proving an auxiliary lemma used in the proof of Corollary~\ref{cor:logconcave}.

\begin{lemma} \label{lem:iso logcon tail}
Let $\delta>0$, $t \geq 1$, let $d$ be a positive integer and let $n = \frac{C t^4 d \log^2(t/\delta)   }{\delta^2}$ for some sufficiently large universal constant $C$.
Let $y_1,\ldots,y_n$ be i.i.d $d$-dimensional vectors from an isotropic log-concave distribution. Then 
$$ \Pr\left[ \left\| \frac{1}{n}\sum y_i \right\| > \delta \right] \leq \exp( -t \sqrt{d}) $$
\end{lemma}
\begin{proof}
For convenience let $S_n = \frac{1}{\sqrt{n}}\sum_{i=1}^n y_i$. Since the $y$'s are independent, $S_n$ is also log-concave distributed. Notice that
$\E[S_n]=0$ and $\E[S_n S_n^T] = \frac{1}{n} \sum \E[y_i y_i^T] = I_d$ hence $S_n$ is isotropic. Now,
$$ \Pr\left[ \left\| \frac{1}{n}\sum y_i \right\| > \delta \right] = \Pr \left[ \|S_n\| > \sqrt{n}\delta \right] = \Pr\left[ \|S_n\| >\sqrt{d} \cdot \sqrt{C t^4 \log^2(t/\delta)} \right] \leq \Pr\left[ \|S_n\| >\sqrt{d} + \sqrt{d} \cdot \frac{1}{2} t \sqrt{C} \right].$$
The last inequality holds for $t\geq 1$ and $C \geq 4$.
It now follows from Theorem~\ref{thm:thinshell} that
$$ \Pr\left[ \left\| \frac{1}{n}\sum y_i \right\| > \delta \right] \leq c_1 \exp(-c_2 t\sqrt{Cd} )$$
where $c_1, c_2$ are some universal constants. Since $t\sqrt{d} \geq 1$, setting $C \geq (\frac{1+\log(c_1)}{c_2})^2$ proves the claim.
\end{proof}

\begin{proof} [Proof of Corollary~\ref{cor:logconcave}]
Let $a = \E[x]$ and let $\tilde{a}=\frac{1}{n} \sum x_i$. Notice that
$$\E[(x-a)(x-a)^T] = \E[xx^T]-\E[x]a^T-a\E[x]+aa^T = I_d-aa^T$$
is a PSD matrix hence $\|a\| \leq 1$. Consider the following equality.
$$ \frac{1}{n} \sum_{i=1}^n (x_i-a) (x_i-a)^\top = \frac{1}{n} \sum_{i=1}^n x_i x_i^\top - a \tilde{a}^\top - \tilde{a} a^\top + aa^\top $$
According to Lemma~\ref{lem:iso logcon tail}, w.p. at least $1-\exp(-t \sqrt{d}) $,
$$ \| \tilde{a} - a \| \leq \delta $$
in  which case, since $\|a\|\leq 1$ and according to the triangle inequality,
$$  \left\| \frac{1}{n} \sum_{i=1}^n x_i x_i^\top  - I_d \right\| \leq \left\| \frac{1}{n} \sum_{i=1}^n (x_i-a) (x_i-a)^\top - (I_d-aa^\top) \right\| + 2\delta $$
According to Theorem~\ref{thm:logconcave}, w.p. at least $1-\exp(-t\sqrt{d})$
$$ \left\| \frac{1}{n} \sum_{i=1}^n (x_i-a) (x_i-a)^\top - (I_d-aa^\top) \right\| \leq \delta$$
and the corollary follows.
\end{proof}

\begin{proof} [Proof of Corollary~\ref{cor:quasi_iso}]
Let $\E[x]=a$. Consider the r.v $y=x-a$. It holds that $\E[y]=0$ and $\E[yy^T] = I_d- aa^T$. Notice that we can derive that 
\begin{equation} \label{eq:a<=1}
\|a\| \leq 1
\end{equation}
As $\E[yy^T]$ is a PSD matrix. Also, it is easy to verify that $y$ is log-concave distributed. We now consider the r.v\footnote{if $I_d-aa^T$ is not of full rank then $y$ is in fact supported in an affine subspace of rank $d-1$ and we can continue the analysis there.} $z = (I_d-aa^t)^{-1/2}y$. It is easy to verify that the distribution of $z$ is also log-concave and isotropic. It follows, from Theorem~\ref{thm:thinshell} that for any $\theta \geq 2$
$$ \Pr\left[\|z\| > \theta \sqrt{d}\right] \leq \Pr\left[\|z\| - \sqrt{d} > \frac{1}{2}\theta \sqrt{d}\right] \leq C'\exp(-c\theta \sqrt{d}) $$
By using equation~\ref{eq:a<=1} we get that for $\theta > 3$
$$ \Pr\left[\|x\| > \theta \sqrt{d}\right] \leq \Pr\left[\|y\| > \theta \sqrt{d}-1\right] \leq \Pr\left[\|z\| > (\theta-1/\sqrt{d}) \sqrt{d}\right] \leq
C'\exp(c'-c'\theta \sqrt{d}) .$$
The last inequality holds since $\theta-1/\sqrt{d} \geq 2$. 
\end{proof}

\end{document}